\documentclass{article}

\usepackage{natbib}
\usepackage[utf8]{inputenc} 
\usepackage[T1]{fontenc}    
\usepackage{hyperref}       
\usepackage{url}            
\usepackage{algorithm}
\usepackage{algorithmic}
\usepackage{booktabs}       
\usepackage{amsfonts}       
\usepackage{amsmath}
\usepackage{amsthm}
\usepackage{caption}
\usepackage{subcaption}
\usepackage{nicefrac}       
\usepackage{microtype}      
\usepackage{xcolor}
\usepackage{graphicx}

\def\Method{{\sc BTGM}}
\def\method{Boltzmann Tuning of Generative Models}

\title{Boltzmann Tuning of Generative Models}

\author{%
  Victor Berger \\
  TAU, LISN,\\
  \small{Inria, Univ. Paris-Saclay, CNRS,}\\
  \texttt{victor.berger@inria.fr} \\
  \and
  Michèle Sebag \\
  TAU, LISN,\\
  \small{CNRS, Univ. Paris-Saclay, Inria,}\\
  \texttt{sebag@lri.fr}
}

\def\objective{criterion}
\def\objectives{criteria}

\newtheorem{lemma}{Lemma}

\begin{document}

\maketitle

\begin{abstract}
The paper focuses on the {\em a posteriori} tuning of a generative model in order to favor the generation of good instances in the sense of some external differentiable \objective.  The proposed approach, called {\em \method} (\Method),  applies to a wide range of applications. It covers conditional generative modelling as a particular case, and offers an affordable alternative to rejection sampling. The contribution of the paper is twofold. Firstly, the objective is formalized and tackled as a well-posed optimization problem; a practical methodology is proposed to choose among the candidate \objectives\ representing the same goal, the one best suited to efficiently learn a tuned generative model. Secondly, the merits of the approach are demonstrated on a real-world application, in the context of robust design for energy policies, showing the ability of \Method\ to sample the extreme regions of the considered \objectives.

\end{abstract}

\section{Introduction}

Deep generative models, including Variational Auto-Encoders (VAEs) \citep{kingma_auto-encoding_2014,rezende_stochastic_2014}, Generative Adversarial Networks (GANs)  \citep{goodfellow_generative_2014}, and Normalizing Flows \citep{rezende_variational_2015},
have been used in a number of ways for (semi)-supervised learning and design. Their usage ranges from robustifying classifiers \citep{kingma_semi-supervised_2014,li_are_2019} to achieving anomaly detection \citep{pidhorskyi_generative_2018,choi_waic_2019}) or solving undetermined inverse problems \citep{ardizzone_analyzing_2019}, from super-resolution of images \citep{ledig_photo-realistic_2017} to computer-assisted creative design \citep{park_gaugan_2019}. In most cases, the fine-tuning of the generative model is seamlessly integrated within the learning process: through the design of the latent representation  \citep{radford_unsupervised_2016,mathieu_disentangling_2016} or through the loss itself, e.g. leveraging labelled information to train conditional generative models \citep{van_den_oord_conditional_2016} (more in section \ref{sec:sota}).

This paper tackles the {\em a posteriori} tuning of a trained generative model, 
aimed at favoring the generation of good samples in the sense of a given 
\objective. The applicative motivation for the proposed approach comes from the design of energy safety policies. In this context, an infrastructure must be tested against a host of diverse production and consumption scenarios, and specifically against their associated consumption peaks.\footnote{These consumption peaks are usually estimated by Monte-Carlo methods, coupling a generative model with rejection sampling, along a tedious and computationnally heavy process, involving the critical estimation of the diversity factor \citep{gonen_electrical_2015,sarfraz_update_2018}.}
  One applicative goal of the proposed approach, called {\em \method} (\Method), is to address this problem by generating consumption curves directly sampling the desired top quantiles of the aggregated consumption distribution. 

This paper considers the general setting defined by a trained generative model and some \objective\ $f$, with the goal of generating  samples biased toward maximizing $f$. This goal is formalized as a constrained optimization problem in the considered distribution space, and a first contribution is to show how to soundly and effectively tackle this problem within the variational inference framework, assuming the differentiability of the \objective\  (section \ref{sec:method}). The proposed \Method\ approach can be applied on the top of any deep generative model, covering conditional generative models \citep{van_den_oord_conditional_2016} as a particular case. It also opens some perspectives in privacy-sensitive domains, e.g. to generate samples in critical and data-poor regions (see also \cite{dash_medical_2020}). In practice, \Method\ offers an affordable and, to our best knowledge, new alternative to rejection sampling. 

Most generally, \Method\ is an attempt toward reconciling data-driven models (here, the generative model learned from extensive data) on the one hand, and analytical, interpretable knowledge (here, the characterization of $f$) on the other hand. While ML traditionally focuses on cases where knowledge/specification is better conveyed through data, some specifications are better conveyed analytically, particularly so when they are poorly illustrated in the data (see also \cite{bessiere_constraint_2017}).
The challenge is to take advantage of both raw data and analytical \objectives\ in an integrated way. 
Along this line, a second contribution of the paper regards how to formulate the user's \objective\ $f$ in the most effective way. Indeed the objective can be formulated in many different ways, up to monotonous transformations of $f$. In order to avoid determining the best formulation of the \objective\ along a tedious trial-and-error phase, an indicator based on the analysis of the underlying optimization process is defined, enabling the comparison of candidate \objectives\ w.r.t. the tuning of the generative model at hand. 

Section \ref{sec:expes} presents several case studies to illustrate the merits and flexibility of the approach: recovering conditional generative modeling (\ref{sec:app-cls}), comparing candidate \objectives\ (\ref{sec:method-compare-f}), showing the flexibility of the approach in the energy consumption modeling domain (\ref{sec:app-elec}) and investigating the {\em a posteriori} deblurring of a generative model (\ref{sec:fine-tune-gm}).



\section{Related work}
\label{sec:sota}
Probability distribution learning is most generally tackled within the Variational Inference (VI) framework. VI being also at the core of the proposed approach, it is presented in section \ref{sec:VI} in a unified way, to both learn a probability distribution from raw data, and tune an existing probability distribution along an analytical \objective. 

The current trends in generative modelling mostly leverage the deep learning efficiency and flexibility to estimate a probability distribution from data, supporting an efficient sampling mechanism \citep{kingma_auto-encoding_2014,rezende_stochastic_2014,goodfellow_generative_2014,rezende_variational_2015}.
Most approaches rely on the introduction of a latent space, whose samples are decoded into a data space.
The generative model is trained to optimize a goodness-of-fit criterion on the original data. In VAEs   \citep{kingma_auto-encoding_2014}, the goodness of fit is the log-likelihood (LL) of the initial data, estimated using the Evidence Lower Bound (ELBO) \citep{bishop_approximating_1998}, as the distribution involves an unknown/unmanageable normalization constant. In GANs \citep{goodfellow_generative_2014}, the goodness of fit criterion is replaced by a 2-sample test, adversarially training the generator and a discriminator estimating whether the generated examples can be discriminated from the original samples. 

\noindent{\bf Distribution spaces.} 
How to make the generative model space flexible enough to accurately approximate the true distribution is mostly handled through using richer latent spaces and/or inference models \citep{burda_importance_2016,van_den_oord_neural_2017,roy_theory_2018,razavi_generating_2019,huang_hierarchical_2019,mathieu_continuous_2019,kalatzis_variational_2020,skopek_mixed-curvature_2020}.
The modelling of multi-mode distributions can also be tackled using continuous and discrete latent variables \citep{jang_categorical_2017,vahdat_dvae_2018}. Specific architectures are designed to exploit the specifics of the data structure, such as Wavenet or Magenta for signal processing \citep{oord_wavenet_2016,roberts_hierarchical_2018} or PixelRNN/CNN for images \citep{oord_pixel_2016,salimans_pixelcnn_2017}, enabling the data likelihood to be explicitly computed and optimized.
Normalizing Flows \citep{rezende_variational_2015,dinh_nice_2015} also proceed by gradually complexifying a distribution, with the particularity that each layer is invertible and enables its Jacobian to be analytically determined, thereby supporting the approximation of the posterior distributions \citep{dinh_density_2017,kingma_improving_2017,ardizzone_analyzing_2019,chen_vflow_2020}.

\noindent{\bf Loss functions.}
The loss function encapsulates the goodness of fit criterion. Many VAE variants focus on the reformulation of the loss to finely control the trade-off between the reconstruction quality and encoding compression \citep{higgins_beta-vae_2017,rezende_taming_2018,alemi_fixing_2018,mathieu_disentangling_2019}. The loss design also aims to avoid pitfalls, notably in terms of instability or mode dropping \citep{arjovsky_wasserstein_2017} with GANs; other distances between the generated and the original distributions 
\citep{nowozin_f-gan_2016,arjovsky_wasserstein_2017} and/or more elaborate 
model architectures \citep{sajjadi_tempered_2018,shaham_singan_2019,torkzadehmahani_dp-cgan_2019} have thus been investigated.

\noindent{\bf Refining Generative Models.} Most generally, the refinement of generative models is based on exploiting supervised information to build conditional models \citep{mirza_conditional_2014,sohn_learning_2015,van_den_oord_conditional_2016,jaiswal_bidirectional_2019}. Another strategy is to use several data samples, within a domain adaptation or multi-task setting \citep{ganin_domain-adversarial_2016}, and to learn coupled generative models \citep{chu_cyclegan_2017}.
Most generally, the customization and refinement of generative models builds upon one or several datasets, exploiting prior knowledge about their features (labels), or about the relationships between the datasets \citep{courty_joint_2017}. 

The alternative explored by \Method\ is to use high-level, analytical information, expressed via \objectives, 
to refine a generative model. On the positive side, this approach is flexible and does not depend on the regions of interest of the instance space to be "sufficiently" represented in the dataset(s). On the negative side, the approach might be {\em too} flexible, in the sense that the regions of interest might be specified in a number of ways, although not all specifications are equally easy to deal with.  We shall return to this point in section \ref{sec:method-compare-f}.

\section{\method}
\def\RR{\mbox{$\mathbb{R}$}}
\label{sec:method}

Let $p$ and $f$ respectively denote the initial generative model defined on the sample space $\mathcal{X} \subset \RR^d$, and the \objective\ of interest ($f: \mathcal{X} \mapsto \RR$). It is assumed wlog that the generative model should be biased toward regions where $f$ takes high values. 
%
The sought biased generative model $q$ is expressed as the solution of a constrained optimization problem: maximizing the expectation of $f$ under $q$, subject to $q$ remaining "sufficiently" close to $p$ in the sense of their Kullback-Leibler divergence:
\begin{equation}
  \mbox{Find~} q = \arg\max \mathbb{E}_q[f] ~s.t.~ D_{KL}(q \|| p) \le C_D
  \label{eq:opt-const}
\end{equation}  
with $C_D$ a positive constant. The Lagrangian $\mathcal{L}$ associated to this primal constrained optimization problem is, with $\lambda$ the Lagrange multiplier accounting for the constraint: 
\begin{equation}
    \mathcal{L}(q) = \int_\mathcal{X} q(x)f(x)dx + \lambda \int_\mathcal{X} q(x) \log\frac{q(x)}{p(x)}dx
\end{equation}
reaching its optimum for:
\begin{equation}
    q_\beta(x) = \frac{1}{Z(\beta)}p(x)e^{\beta f(x)} \label{eq:q_beta}
\end{equation}
with $\beta = 1/\lambda$ and $Z(\beta)$ the normalization constant. \Method\ tackles the dual optimization problem of minimizing $D_{KL}(q||p)$ subject to $\mathbb{E}_q f$ being greater than some constant $C_f$, yielding solution $q_\beta$ for some $\beta$ depending on $C_f$ (below): 
\begin{equation}
    q_\beta = \mathop{\arg\max}_{q} \; \beta \, \mathbb{E}_q f - D_{KL}(q\|p) \label{eq:opt_q}
\end{equation}

\subsection{Finding \texorpdfstring{$\beta$}{beta}}\label{sec:Pareto}
Varying the strength of the bias, from no bias ($\beta = 0$ yields $q_\beta = p$) to $\beta = \infty$ (with   $q_\beta$ with support in the optima of $f$) yields a family of distributions, the Pareto front associated to the maximization of $\mathbb{E}_q f$ and minimization of $D_{KL}(q\|p)$. Simple calculations yield (Appendix \ref{app:derivatives}): 
\begin{equation}
    \frac{d}{d \beta} D_{KL}(q_\beta \| p) = \beta Var_{q_\beta}(f) \quad \text{and} \quad \frac{d}{d \beta} \mathbb{E}_{q_\beta} f = Var_{q_\beta}(f) \label{eq:d_beta}
\end{equation}
$D_{KL}(q_\beta\|p)$ and $\mathbb{E}_{q_\beta} f$ being strictly increasing functions of $\beta$, there exists a one-to-one mapping between the values of $D_{KL}(q_\beta\|p)$, and $\mathbb{E}_{q_\beta}f$, hence 
there exists a single $\beta$ value such that $q_\beta$ solves the constrained optimization problem.  
Further calculations yield the second order derivatives: 
\begin{equation}
\label{eq:d_beta2}
    \frac{d^2}{d \beta^2} D_{KL}(q_\beta \| p) = Var_{q_\beta}(f) + \beta \mathbb{E}_{q_\beta} \left( f - \mathbb{E}_{q_\beta} f\right)^3 \quad
    \frac{d^2}{d \beta^2} \mathbb{E}_{q_\beta} f = \mathbb{E}_{q_\beta} \left( f - \mathbb{E}_{q_\beta} f\right)^3
\end{equation}
Note that any generative model $q_\beta$ enables by construction to empirically estimate the first three moments of $f$ under $q_\beta$, as well as  $D_{KL}(q_\beta \|| p)$. Plugging these estimates in Eqs. \ref{eq:d_beta} and \ref{eq:d_beta2} and using second order optimization methods \citep{boyd_convex_2004} enables to quickly converge toward the desired value of $\beta$, i.e. such that 
$\mathbb{E}_{q_\beta}f = C_f$  (Alg. \ref{alg:main}).

\begin{algorithm}
	\caption{\Method}
	\label{alg:main}
	\begin{algorithmic}
		\STATE $\beta \gets 0$
		\REPEAT
			\STATE $q_\beta \gets \mathop{\arg\max}_q \: \beta \mathbb{E}_q f - D_{KL}(q\|p)$ (section 3.2)
			\STATE Estimate $D_{KL}(q_\beta\|p)$, $\mathbb{E}_{q_\beta} f$, $Var_{q_\beta}(f)$ and $\mathbb{E}_{q_\beta} \left( f - \mathbb{E}_{q_\beta} f\right)^3$ by Monte-Carlo sampling
			\STATE Do a second-order update of $\beta$ using Eq. \ref{eq:d_beta} and \ref{eq:d_beta2}.
		\UNTIL{convergence of $\beta$}
	\end{algorithmic}
\end{algorithm}

\subsection{Building \texorpdfstring{$q_\beta$}{q-beta}}
\label{sec:VI}
It is seen that Eq. \ref{eq:opt_q} essentially defines a Variational Inference (VI) problem for each  $\beta$ value. This problem is reformulated using the Evidence Lower Bound (ELBO) \citep{bishop_approximating_1998}:
\begin{equation}
    q_\beta = \mathop{\arg\max}_{q} \; H(q) + \mathop{\mathbb{E}}_{x\sim q} \left[ \beta f(x) + \log p(x) \right] \label{eq:q_vi}
\end{equation}
with $H(q)$ the entropy of $q$.

VI is intensively used for generative modelling, optimizing $q$ based on samples of the true distribution. The optimization of the ELBO \citep{ranganath_black_2014} classically proceeds by leveraging stochastic optimization \citep{hoffman_stochastic_2013} or building upon the reparametrization trick \citep{kingma_auto-encoding_2014}. The distribution space is chosen to efficiently approximate the posterior beyond the mean-field approximation, using low-rank Gaussian distributions \citep{ong_gaussian_2018}, mixtures of Gaussian distributions \citep{gershman_nonparametric_2012}, or mixtures of an arbitrary number of distributions via boosting methods \citep{guo_boosting_2017,miller_variational_2017}. 
An alternative is offered by Normalizing Flows, where the neural architecture achieves an invertible transformation enabling its Jacobian to be analytically determined, thereby supporting the approximation of the posterior distributions \citep{rezende_variational_2015,kingma_improving_2017}.
The use of stochastic equations such as Langevin Monte-Carlo \citep{welling_bayesian_2011} can also be used to directly sample from the target distribution, without explicitly modelling it beforehand.

In the considered context, VI is used to tune an existing $q$ after $f$. Note that $q_\beta$ mostly specializes the initial generative model $p$ (as opposed to, exploring the very low probability regions of $p$, which would significantly degrade  $D_{KL}(q_\beta\|p)$). Therefore $q_\beta$ will expectedly have its typical set \citep{nalisnick_detecting_2019} roughly included in the typical set of $p$. 
Accordingly, $q$ is sought via deterministically perturbing the samples drawn according to $p$ (returning
$ x = g(\hat x)$
with $\hat x$ sampled from $p$ and $g: \mathcal{X} \mapsto \mathcal{X}$ the perturbation). The Normalized Flow neural architecture\footnote{The study of other neural architectures is left for further work.} is used to find $g$, for it makes its Jacobian explicit and easy to compute its determinant. With $J(g)$ the Jacobian matrix of $g$, it comes: 
\begin{equation}
    q(x) = p(\hat{x}) \left|J(g)(\hat{x})\right|^{-1}
\end{equation}
The optimization problem (Eq. \ref{eq:q_vi}) then reads: 
\begin{equation}\mbox{Find ~} g = 
    \mathop{\arg\max}_g \mathop{\mathbb{E}}_{\hat{x}\sim p} \left[\beta f(g(\hat{x})) + \log p(g(\hat{x})) + \log \left|J(g)(\hat{x})\right| \right]
\end{equation}
Some care is exercised at the initialization of Algorithm \ref{alg:main}, setting $g$ very close to identity; the subsequent iterations proceed by warm-start, setting $g_i$ to the $g_{i-1}$ learned in the previous iteration. 

\subsection{Operating in latent space}
\label{sec:apply-to-latent}
The use of latent space is pervasive in generative modelling, notably for the sake of dimensionality reduction.
The samples in the latent space (drawn after some simple, usually Gaussian, prior distribution $p(z)$) are mapped onto the instance space by the decoder module $p(x|z)$, in a deterministic ($x = \mathop{\mathbb{E}} dec(z)$) or probabilistic ($x \sim dec(z)$) way. The generative model is $p(x) = \int_z p(x|z)p(z)dz$.\footnote{Note that in the VAE case, $p(x|z)$ can be considered as a quasi deterministic distribution when using an observation model with small variance.}

Most interestingly, \Method\ can operate in the latent space too, tuning the latent distribution $p(z)$ and yielding a tuned latent distribution noted  $q_\beta(z)$. The sought tuned distribution $q_\beta(x)$ in the instance space is derived from $q_\beta(z)$ through the decoder module:
$$q_\beta(x) = \int_z p(x|z)q_\beta(z)dz $$

Operating on the latent space with a frozen decoder module offers several advantages. Firstly, the optimization criterion remains well defined, with 
$D_{KL}(q_\beta(z)\|p(z))$ an upper bound of $D_{KL}(q_\beta(x)\|p(x))$ (Appendix \ref{app:kl-bound}). 
Secondly, conducting the optimization process in the latent space is easier and yields more robust results, due to the dimension of the latent space being usually lower than that of $\mathcal{X}$ by one or several orders of magnitude, and the generative distribution $p(z)$ being usually a simple one, e.g. $\mathcal{N}(0; Id)$. 
Last, freezing the decoder ensures that the support of the eventual generative model remains included in the support of the initial one.
Formally, applying \Method\ in the latent space amounts to replacing \objective\ $f$ by $\hat{f}$ defined on the latent space as\footnote{If $p(x|z)$ is deterministic or has a low variance, the expectation can be well approximated by a single sample.}:
\begin{equation}
    \label{eq:alt-f}
    \hat{f}(z) = \mathbb{E}_{x \sim p(x|z)} f(x)
\end{equation}

\section{Case studies}
\label{sec:expes}
This section reports on four case studies conducted with \Method. The code is available in supplementary material. 

\def\fl{\mbox{$f_{log.h}$}}
\def\fh{\mbox{$f_h$}}

\subsection{Conditional generative modelling}
\label{sec:app-cls}
\Method\ covers conditional generative modelling as a particular case. 
In a supervised learning context, with $h$ an (independently trained) classifier 
and $h(\ell|x)$ the probability of $x$ to be labelled as $\ell$, 
let \objective\ $f$ be set to $\log h(\ell|x)$ in  order to bias the generative model toward class $\ell$. Model $q_\beta$ reads:
\begin{equation}
    q_\beta(x) \propto p(x) h(\ell|x)^\beta
\end{equation}
defining a standard conditional generative model of class $\ell$ for $\beta = 1$ (assuming that $h(\ell|x)$ accurately estimates $p(\ell|x)$). Through parameter $\beta$, one can also control the fraction of samples closest to class $\ell$, by setting the constraint $D_{KL}(q_\beta\|p) \le - log(\rho)$ with $\rho$ the mass of the desired fraction (Fig. \ref{fig:mnist-cls}, Appendix \ref{app:bold-mnist}). 
In the same spirit, \Method\ can be used to debug classifier $h$, e.g. by generating samples in the ambiguous regions at the frontier of two or several classes (e.g. using as \objective $f$  the probability of the second most probable class or the entropy of the prediction of the classifier), and inspecting $h$ behavior in this region. 

\subsection{Assessing \objectives\ {\em ex ante}}
\label{sec:method-compare-f}
As said, an \objective\ $f$ can be represented in a number of ways, e.g. considering all $g \circ f$ with $g$ a monotonous function; still, the associated optimization problems (Eq. \ref{eq:q_vi}) are  in general of varying difficulty. In order to facilitate the usage of \Method\ and avoid a tedious trials and errors phase, 
some way of comparing {\em a priori} two \objectives\ is thus desirable.

It is easy to see that the Pareto front of \Method\ solutions (section \ref{sec:Pareto}) is invariant under affine transformations\footnote{The addition of a constant is cancelled out by the normalisation constant of $q_\beta$, and a multiplicative transform resulting in choosing another $\beta$ value.} of $f$. In the following, any \objective\ $f$ is normalized via an affine transformation (below), yielding an expectation and variance under $p$  respectively set to $0$ and $1$.

Informally, the difficulty of the optimization problem reflects how much $p$ has to be transformed to match $q_\beta$. This difficulty can be quantified from the log ratio of $p$ and $q_\beta$, specifically measuring whether this log ratio is subject to fast variations. A measure of difficulty thus is the norm of $\nabla_x \log \frac{q_\beta(x)}{p(x)}$. Note that the distribution of this gradient norm can be empirically estimated:
\begin{equation}
    \nabla_x \log \frac{q_\beta(x)}{p(x)} = \beta \nabla_x f(x)
\end{equation}
Overall, samples generated from $p$ are used:  i/ to normalize the candidate \objectives; ii/ to estimate the distribution of their gradient norm; and iii/ to compare two \objectives\ and prefer the one with more regular distribution,  as defining a smoother optimization problem. 
This analysis extends to the tuning of generative models in latent space, replacing $f$ with $\hat{f}$ (Eq. \ref{eq:alt-f}).

The methodology is illustrated in the conditional modelling context (section \ref{sec:app-cls}), to compare the two \objectives\ $f(x) = h(\ell|x)$ and $f(x) = log (h(\ell|x))$, respectively referred to as \fl\ and \fh. The distribution of their gradients under $p$ is displayed on Fig. \ref{fig:mnist-cls-grad}. The binned distribution of the gradient norms in latent space for all ten classes, is estimated from 10,000 samples (truncated for readability: the highest values for the gradient norm of \fh\ go up to 60-100, to be compared to 10-15 for the gradient norm of \fl).

\begin{figure}
    \centering
    \includegraphics[width=\columnwidth]{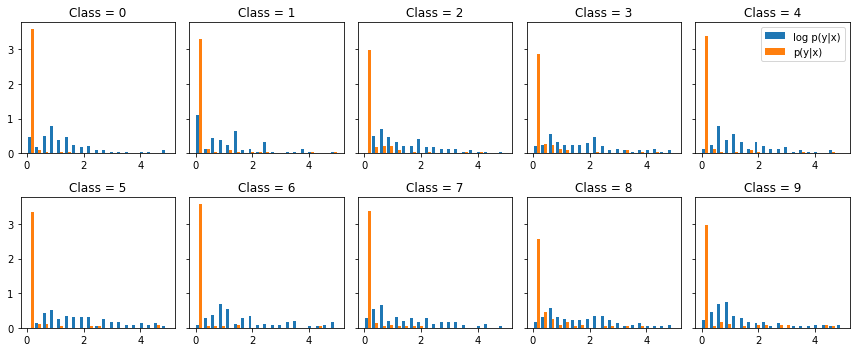}
    \caption{Comparing \objectives \fl\ (in blue) and \fh\ (in orange) on MNIST: binned distribution of their gradient norms (better seen in color). The distribution tails are truncated for the sake of visualization, see text.}
    \label{fig:mnist-cls-grad}
\end{figure}
The distribution of the \fh\ gradient norm shows a high mass on 0 with quite some high values, suggesting a complex optimization landscape with a number of plateaus (gradient norm 0) separated by sharp boundaries (high gradient norms). In opposition, the distribution of the \fl\ gradient norm is flatter with a more compact support, suggesting a manageable optimization landscape where the gradient offers some (bounded) information in most regions. 
Accordingly, it is suggested  \fl\ is much more amenable to the tuning of the generative model than \fh, which is empirically confirmed (Appendix \ref{app:grad}). Overall, the proposed methodology allows to efficiently and inexpensively compare {\em a priori}  candidate \objectives, and retain the most convenient one.

\subsection{A real-world case study}
\label{sec:app-elec}
This section focuses on using \Method\ as an alternative to rejection sampling on the real-world problem of smart grid energy management and dimensioning. For the sake of reproducibility, an experiment on MNIST along the same rejection sampling ideas is detailed in Appendix \ref{app:bold-mnist}.

The goal is to sample the extreme energy consumption aggregated curves under a number of usage scenarii (e.g. traffic schedules, localisation of electric car charging stations, telecommuting and its prevalence), to estimate the peak consumption.
The aggregation of multiple consumers into a single consumption curve tends to smooth the consumption peak, 
as measured by the so-called diversity factor \citep{sarfraz_update_2018}. The difficulty is that the relationship between the aggregated and the individual consumption curves is ill-known, essentially studied by Monte-Carlo sampling, making it desirable to design a flexible generative model of aggregated consumption curves. 

In a preliminary phase, a VAE is trained on weekly consumption curves to model the aggregated consumption of 10 households (Fig. \ref{fig:elec-data} and \ref{fig:elec-p}).
A first \objective\ $f_1$ considers the maximum consumption reached over the week, with the aim to sample the $1\%$ top quantile of the curves (yielding $C_D = -\log {10^{-2}} = 4.61$).
The tuned generative model (Fig. \ref{fig:elec-q} and \ref{fig:elec-mean-std}) sample curves with a significantly higher peak consumption; note that these curves have a high weekly consumption, too. Indeed, the generative model makes it more likely to reach a high peak during a high consumption week than in an average consumption week (e.g. due to external factors such as cold weather). The freezing of the decoder enables to preserve the plausibility of the generated samples, while sampling in the extreme regions of the distribution according to $f_1$.

A second \objective\ $f_2$, concerned with maximizing the difference between the mean consumption on Wednesdays and the mean consumption over the whole week, is considered to illustrate the versatility of \Method\ (Fig. \ref{fig:elec-wed-q} and \ref{fig:elec-wed-mean-std}). Other choices of $f$ are  discussed in Appendix \ref{app:electrical}.

\begin{figure}
    \centering
    \begin{subfigure}[t]{0.48\textwidth}
        \includegraphics[width=\columnwidth]{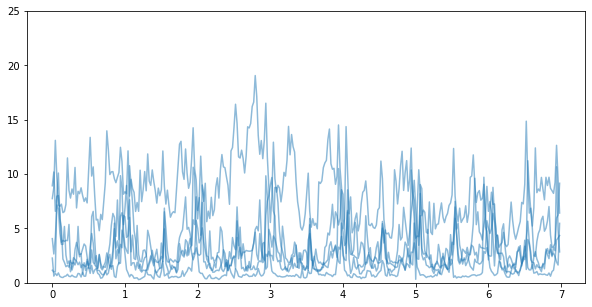}
        \caption{5 true consumption curves.}
        \label{fig:elec-data}
    \end{subfigure}
    \hfill
    \begin{subfigure}[t]{0.48\textwidth}
        \includegraphics[width=\columnwidth]{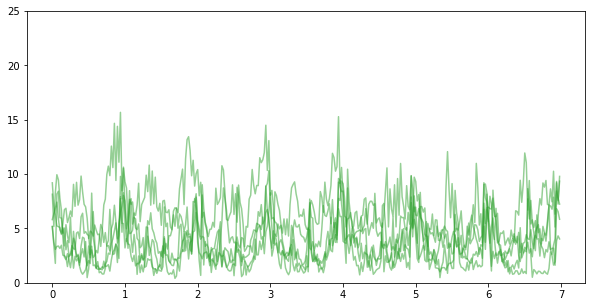}
        \caption{5 samples from the unbiased model $p$.}
        \label{fig:elec-p}
    \end{subfigure}
    \begin{subfigure}[t]{0.48\textwidth}
        \includegraphics[width=\columnwidth]{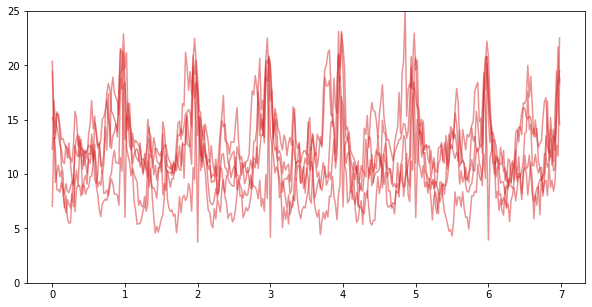}
        \caption{5 samples from the model $q_\beta$ tuned to maximize peak consumption.}
        \label{fig:elec-q}
    \end{subfigure}
    \hfill
    \begin{subfigure}[t]{0.48\textwidth}
        \includegraphics[width=\columnwidth]{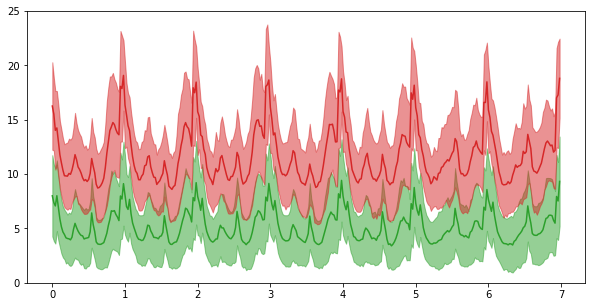}
        \caption{Mean and standard deviation of samples generated after $p$ and $q_\beta$ tuned to maximize peak consumption.}
        \label{fig:elec-mean-std}
    \end{subfigure}
    \begin{subfigure}[t]{0.48\textwidth}
        \includegraphics[width=\columnwidth]{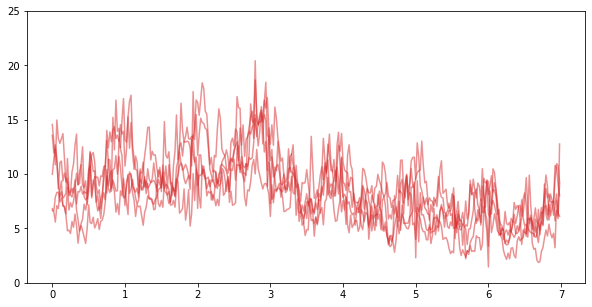}
        \caption{5 samples from the tuned model $q_\beta$ tuned to maximize Wednesday's consumption only.}
        \label{fig:elec-wed-q}
    \end{subfigure}
    \hfill
    \begin{subfigure}[t]{0.48\textwidth}
        \includegraphics[width=\columnwidth]{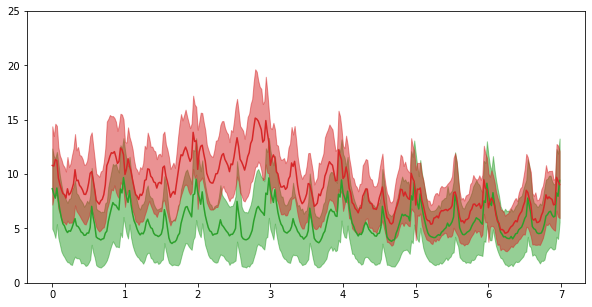}
        \caption{Mean and standard deviation of samples generated after $p$ and $q_\beta$ tuned to maximize Wednesday's consumption only.}
        \label{fig:elec-wed-mean-std}
    \end{subfigure}
    \caption{Applying \Method\ to tune the generation of weekly energy consumption curves, reporting the consumption (in kW on $y$ axis) vs the day (on $x$ axis). \\Top: real sample curves (a) and $p$-generated samples, with $p$ the initial VAE model (b). \\Middle row: tuning $p$ toward top 1\% weekly energy consumption curves (\objective $f_1$); tuned generated samples (c), and comparison of  $p$ with $p$ tuned after $f_1$ (d).\\ Bottom row:  tuning $p$ toward top 1\% Wednesday energy consumption curves (\objective $f_2$, see text);   tuned generated samples (e), and comparison of $p$ with $p$ tuned after $f_2$ (f). The VAE $p$, composed of encoder and decoder modules with 10 blocks of residual networks each, is trained from ca 8 million weekly consumption curves; the mean and deviation of the initial and tuned generative models are computed over 1,000 samples. Better seen in color.}
\end{figure}

\subsection{Refining a generative model {\em a posteriori}}
\label{sec:fine-tune-gm}

Another potential usage of \Method\ is to refine existing generative models, e.g. preventing a VAE from generating out-of-distribution samples \citep{arjovsky_towards_2017}. 
\def\pv{\mbox{$p_{VAE}$}}
Let \pv\ denote an overly general generative model, and let $f$ be defined as a adversarial classifier, discriminating the generated samples from the true data distribution $p_D$. When converged and in the large sample limit, the discriminator yields an estimation of $\frac{p_D(x)}{p_D(x) + p_{VAE}(x)}$ \citep{goodfellow_generative_2014}. 

When using \objective\ $f(x) = \log\frac{p_D(x)}{p_{VAE}(x)}$, given by the pre-activation output of the discriminator, to tune model \pv, one gets the generative model $q_\beta$ defined as: 
\begin{equation}
    q_\beta(x) \propto p_{VAE}(x)^{1-\beta}p_D(x)^\beta
\end{equation}
In this scheme, \Method\ aims to actually draw the generative model closer to the true distribution $p_D$. Compared to the mainstream GAN scheme, the difference is that the discriminator is used {\em a posteriori}: the generative modelling is decoupled from its adversarial tuning and the concurrent training procedure is replaced by the sequence of two (comparatively straightforward) optimization procedures, firstly training $\pv$ and secondly tuning it toward $f$. Results illustrating the proposed methodology are presented in Appendix \ref{app:adversarial}.
This sequential adversarial generative modelling relies on two interdependent assumptions. Firstly, \pv\ must  be able to accurately reconstruct the whole training dataset;
more precisely, the support of distribution \pv\ must cover that of the data distribution $p_D$. Secondly, the discriminator needs be not saturated and give highly-confident predictions, for its gradient to provide sufficient information to refine \pv\ (this also requires the former assumption to hold).

\section{Discussion and Perspectives}
\label{sec:discu}
The contribution of the paper is a new theoretical formulation and algorithm for the {\em a posteriori} refinement of a wide class of generative models,  including GANs, VAEs, and explicit likelihood models. When the considered generative model relies on the use of a latent space, \Method\ can operate directly in  the latent space, favoring the scalability of the approach w.r.t. high-dimensional spaces. 
\Method\ offers a new alternative to rejection sampling in order to explore the extreme quantiles of the data distribution w.r.t. any \objective\ $f$, subject to $f$ being differentiable. The proof of concept presented in the domain of energy management, where the consumption peak is estimated from the extreme quantiles of the consumption curves, is to our best knowledge the first and only alternative to rejection sampling in this context. 

Three perspectives for further work are considered. In the short term, a first goal is to use \Method\ to better understand when and why the dropping phenomenon occurs in the adversarial setting. On-going results show that a VAE model can indeed be refined {\em a posteriori} using a discriminator as \objective $f$; however, it is observed that mode dropping does appear when the pressure on $f$ is increased beyond a certain level. In order to avoid this loss of diversity, a research perspective is to extend \Method\ to the general multi-criteria optimization setting, tuning the considered generative models with several criteria (e.g. the discriminator $f$, and the lequi-distribution of the classes).

A second perspective is to use \Method\ in the context of privacy-sensitive data. The use of generative models for releasing non-sensitive though realistic samples has been explored \citep{torkzadehmahani_dp-cgan_2019,long_scalable_2019,augenstein_generative_2020}. \Method\ makes it feasible to train a model from large datasets (thus offering a better model with better privacy guarantees) and focus it {\em a posteriori} on the target of interest, e.g. a rare mode of a disease. The eventual biased generative model will expectedly both inherit the privacy guarantees of the general model, and yield the focused samples as desired.

Another perspective is to extend \Method\ in the direction of Bayesian Optimization \citep{mockus_application_1978,rasmussen_gaussian_2004}, and Interactive Preference Learning pioneered by \citep{brochu_tutorial_2010,viappiani_recommendation_2011}. Specifically in the context of Optimal Design, the expert-in-the-loop setting can be leveraged to alternatively bias the generative model toward the experts' preferences, and learn a model of their preferences. While facing the challenges of interactive preference learning,  this approach would pave the way toward a focused augmentation of the data, under the experts' control.

\bibliographystyle{abbrvnatnourl}
\bibliography{bibliography}

\newpage

\appendix

\section{Closed form derivatives of \texorpdfstring{$D_{KL}(q\|p)$}{DKL(q||p)} and \texorpdfstring{$\mathbb{E}_q f$}{E\_q f}}

\label{app:derivatives}

From Eqs. (1-2)
$$    q_\beta(x) = \arg\min  \int_\mathcal{X} q(x)f(x)dx + \frac{1}{\beta} \int_\mathcal{X} q(x) \log\frac{q(x)}{p(x)}dx$$
it follows:
\begin{equation}
    q_\beta(x) = \frac{1}{Z(\beta)} p(x) e^{\beta f(x)} \label{eq:q_def}
\end{equation}
with normalization constant $Z(\beta)$ ensuring that $q_\beta$ is a probability distribution. The derivatives of $D_{KL}(q\|p)$ and $\mathbb{E}_q f$ follow from Lemmas 1 and 2.

\begin{lemma}
The derivative $\frac{d}{d \beta} \log Z(\beta)$ reads: 
\begin{equation}
\label{eq:d_logz}
{\frac{d}{d \beta} \log Z(\beta) = \mathbb{E}_{q_\beta} f}
\end{equation}
\end{lemma}

\begin{proof}
As $Z(\beta) = \int_x p(x) e^{\beta f(x)} dx$ by definition, it follows: 

\begin{equation}
    \begin{aligned}
        \frac{d}{d \beta} \log Z(\beta) &= \frac{1}{Z(\beta)} \frac{d}{d \beta} Z(\beta) \\
        &= \frac{1}{Z(\beta)} \frac{d}{d \beta} \int_{\cal X} p(x) e^{\beta f(x)} dx \\
        &= \frac{1}{Z(\beta)} \int_{\cal X} f(x) p(x) e^{\beta f(x)} dx \\
        &= \int_x f(x) q_\beta(x) \\
        &= \mathbb{E}_{q_\beta} f
    \end{aligned}
\end{equation}
\end{proof}

\begin{lemma}
Let $h: \mathcal{X} \rightarrow \mathbb{R}$ be a function (possibly depending on $\beta$). The derivative of its expectation on $q_\beta$ wrt $\beta$ reads:

\begin{equation}
    \label{eq:d_e_h}
    {\frac{d}{d \beta} \mathbb{E}_{q_\beta} h = \mathbb{E}_{q\beta}\left[fh + \frac{\partial h}{\partial \beta} \right] - \left(\mathbb{E}_{q_\beta} f\right) \left(\mathbb{E}_{q_\beta} h\right)}
\end{equation}
\end{lemma}

\begin{proof}
\begin{equation}
\begin{aligned}
\frac{d}{d \beta} \mathbb{E}_{q_\beta} h &= \frac{d}{d \beta} \frac{1}{Z(\beta)}\int_x h(x) p(x) e^{\beta f(x)} dx \\
&= \frac{1}{Z(\beta)} \int_x \left(h(x)f(x) + \frac{\partial h}{\partial \beta}(x)\right) p(x) e^{\beta f(x)} dx \\
& \quad - \frac{1}{Z(\beta)^2} \frac{d Z}{d \beta} \int_x h(x) p(x) e^{\beta f(x)} dx \\
&= \mathbb{E}_{q_\beta} \left[ hf + \frac{\partial h}{\partial \beta} \right] - \left(\mathbb{E}_{q_\beta} h \right) \frac{d}{d \beta} \log Z(\beta) \\
&= \mathbb{E}_{q\beta}\left[fh + \frac{\partial h}{\partial \beta} \right] - \left(\mathbb{E}_{q_\beta} f\right) \left(\mathbb{E}_{q_\beta} h\right)
\end{aligned}
\end{equation}
\end{proof}

Lemmas 1 and 2 yield the first and second derivatives of $\mathbb{E}_{q_\beta} f$.

\begin{lemma}
The first and second derivatives of $\mathbb{E}_{q_\beta} f$ wrt $\beta$ read:

\begin{equation}
    {\frac{d}{d \beta} \mathbb{E}_{q_\beta} f = Var_{q_\beta} f} \quad \text{and} \quad {\frac{d^2}{d \beta^2} \mathbb{E}_{q_\beta} f = \mathbb{E}_{q_\beta} \left( f - \mathbb{E}_{q_\beta} f\right)^3 }
\end{equation}
\end{lemma}

\begin{proof}
Replacing $h$ with $f$ in Eq. \ref{eq:d_e_h}, and noting that $f$ does not depend on $\beta$, yields the first derivative:
\begin{equation}
    \frac{d}{d \beta} \mathbb{E}_{q_\beta} f = \mathbb{E}_{q_\beta} f^2 - \left(\mathbb{E}_{q_\beta}f \right)^2 = Var_{q_\beta} f
\end{equation}
Noting that $Var_{q_\beta} f = \mathbb{E}_{q_\beta}\left( f - \mathbb{E}_{q_\beta} f\right)^2$ and replacing $h$ with $\left(f - \mathbb{E}_{q_\beta} f\right)^2$ (that does depend on $\beta$) in Eq. \ref{eq:d_e_h} yields the second derivative:

\begin{equation}
\begin{aligned}
    \frac{d^2}{d \beta^2} \mathbb{E}_{q_\beta} f &= \frac{d}{d \beta} \mathbb{E}_{q_\beta}\left( f - \mathbb{E}_{q_\beta} f\right)^2\\
    &= \mathbb{E}_{q_\beta} \left[ f \left( f - \mathbb{E}_{q_\beta} f\right)^2 - 2\left( f - \mathbb{E}_{q_\beta} f\right) \frac{d}{d \beta} \mathbb{E}_{q_\beta} f \right] - \left(\mathbb{E}_{q_\beta} f\right)\left(\mathbb{E}_{q_\beta}\left( f - \mathbb{E}_{q_\beta} f\right)^2\right) \\
    &= \mathbb{E}_{q_\beta} \left( f - \mathbb{E}_{q_\beta} f\right)^3 - 2 \underbrace{\mathbb{E}_{q_\beta} \left[ f - \mathbb{E}_{q_\beta} f\right]}_{= 0} \frac{d}{d \beta} \mathbb{E}_{q_\beta} f \\
    &= \mathbb{E}_{q_\beta} \left( f - \mathbb{E}_{q_\beta} f\right)^3
\end{aligned}
\end{equation}
\end{proof}

Lemmas 1 and 2 likewise yield the first and second derivatives of $D_{KL}(q_\beta \| p)$:

\begin{lemma}
The first and second derivatives of $\mathbb{E}_{q_\beta} f$ wrt $\beta$ read:

\begin{equation}
    {\frac{d}{d \beta} D_{KL}(q_\beta \| p) = \beta Var_{q_\beta} f} \quad \text{and} \quad {\frac{d^2}{d \beta^2} D_{KL}(q_\beta \| p) = Var_{q_\beta} f + \beta \mathbb{E}_{q_\beta} \left( f - \mathbb{E}_{q_\beta} f\right)^3}
\end{equation}
\end{lemma}

\begin{proof}
By definition:

\begin{equation}
    \begin{aligned}
        D_{KL}(q_\beta \| p) &= \mathbb{E}_{q_\beta} \log \frac{q_\beta}{p} \\
        &= \mathbb{E}_{q_\beta} \left[ \beta f - \log Z(\beta) \right] \\
        &= \beta \mathbb{E}_{q_\beta} \left[ f \right] - \log Z(\beta)
    \end{aligned}
\end{equation}

Lemmas 1 and 2 thus yield: 

\begin{equation}
    \begin{aligned}
        \frac{d}{d \beta} D_{KL}(q_\beta \| p) &= \mathbb{E}_{q_\beta} \left[ f \right] + \beta \frac{d}{d \beta} \mathbb{E}_{q_\beta} \left[ f \right] - \frac{d}{d \beta} \log Z(\beta) \\
        &= \mathbb{E}_{q_\beta} \left[ f \right] + \beta Var_{q_\beta} \left[f\right] - \mathbb{E}_{q_\beta} \left[ f \right] \\
        &= \beta Var_{q_\beta} f
    \end{aligned}
\end{equation}
and:
\begin{equation}
    \begin{aligned}
        \frac{d}{d \beta} D_{KL}(q_\beta \| p) &= Var_{q_\beta} f + \beta \frac{d}{d \beta} Var_{q_\beta} f \\
        &= Var_{q_\beta} f + \beta \mathbb{E}_{q_\beta} \left( f - \mathbb{E}_{q_\beta} f\right)^3
    \end{aligned}
\end{equation}
which concludes the proof.
\end{proof}

\section{Bounding \texorpdfstring{$KL(q\|p)$}{DKL(q||p)} on latent space}

\label{app:kl-bound}

\begin{lemma}
Let $p(x,z)$ be a generative model built on a sampling of a latent space ($p(x,z) = p(z)p(x|z)$, with $p(x|z)$ the decoder mapping the latent onto the instance space). Let generative model $q(x,z)$ be defined as $q(x,z) = q(z)p(x|z)$ (freezing the decoder and modifying the latent distribution). Then: 
\begin{equation}
    {D_{KL}(q(x)\|p(x)) \leq D_{KL}(q(z)\|p(z))}
\end{equation}
\end{lemma}

\begin{proof}
It is seen that, for any two distributions $q$ and $p$ of two variables, the Kullback-Leibler divergence between their marginals is always smaller than the Kullback-Leibler divergence between the full distributions:

\begin{equation}
\begin{aligned}
    D_{KL}(q(a, b)\|p(a,b)) &= \mathbb{E}_q \log \frac{q(a,b)}{p(a,b)}\\
    &= \mathbb{E}_q \log \frac{q(a)q(b|a)}{p(a)p(b|a)}\\
    &= D_{KL}(q(a)\|p(a)) + \mathbb{E}_q D_{KL}(q(b|a)\|p(b|a))\\
    &\geq D_{KL}(q(a)\|p(a))\\
\end{aligned}
\end{equation}
Replacing $p(a,b)$ with $p(x,z) = p(z)p(x|z)$ (respectively, $q(a,b)$ with $q(x,z) = q(z)p(x|z)$) yields:

\begin{equation}
\begin{aligned}
    D_{KL}(q(x)\|p(x)) &\leq D_{KL}(q(x,z)\|p(x,z)) \\
    &\leq D_{KL}(q(z)\|p(z)) + \mathbb{E}_{q} \underbrace{D_{KL}(p(x|z)\|p(x|z))}_{=0} \\
    &\leq D_{KL}(q(z)\|p(z))
\end{aligned}
\end{equation}
\end{proof}

\section{Comparing two criteria: detailed analysis}
\label{app:grad}

As the intended bias can be expressed using different criteria, the question of comparing these (based on the distribution of their gradient norms) was discussed in section 4.2. Complementary experiments are conducted as follows, along the same setting aimed to conditionalize generative model $p$ using a classifier $p(\ell|x)$. 

A first remark is that the closed form values of $\mathbb{E}_{q_\beta} f$ and $D_{KL}(q_\beta\|p)$ can be estimated using samples from $p$. Specifically, expectations under $q_\beta$ can be reframed as expectations under $p$:

\begin{equation}
    Z(\beta) = \int_{\mathcal{X}} p(x) e^{\beta f(x)} dx = \mathbb{E}_p \left[ e^{\beta f} \right]
\end{equation}

\begin{equation}
    \mathbb{E}_{q_\beta} f = \int_{\mathcal{X}} f(x) \frac{p(x) e^{\beta f(x)}}{Z(\beta)} dx = \frac{\mathbb{E}_p\left[f e^{\beta f}\right]}{\mathbb{E}_p\left[ e^{\beta f} \right]}
    \label{eq:Eq}
\end{equation}

\begin{equation}
    D_{KL}(q_\beta \| p) = \beta \mathbb{E}_{q_\beta} \left[f\right] - \log Z(\beta) = \frac{\mathbb{E}_p\left[f e^{\beta f}\right]}{\mathbb{E}_p\left[ e^{\beta f} \right]} - \log \mathbb{E}_p\left[ e^{\beta f} \right]
    \label{eq:DKL}
\end{equation}

Eqs. \ref{eq:Eq}-\ref{eq:DKL} enable to estimate the closed form values of $\mathbb{E}_{q_\beta} f$ and $D_{KL}(q_\beta\|p)$ {\em vs} $\beta$, using samples drawn after $p$. The comparison of these estimates with the actual $\mathbb{E}_{\hat q_\beta} f$ and $D_{KL}(\hat q_\beta\|p)$
indicates how well \Method\ is dealing with the considered criterion. 

In the considered example, one wants to compare both criteria $f_h$ and $f_{\log h}$, respectively defined as $f_h(x) = p(\ell|x)$ and $f_{\log h}(x) = log p(\ell|x)$. The discrepancy between the theoretical estimate and the actual estimate is displayed on Fig. \ref{fig:th-grad-fh} for $f_h$ (respectively Fig. \ref{fig:th-grad-flogh} for $f_{\log h}$). The same optimization procedure was used in both cases, targeting the class $\ell = 4$.

\begin{figure}
    \centering
    \begin{subfigure}[t]{0.3\textwidth}
        \includegraphics[width=\columnwidth]{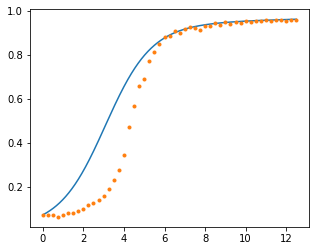}
        \caption{$\mathbb{E}_{q_\beta}\,f_h$ ($y$ axis) vs $\beta$ ($x$ axis).}
        \label{fig:th-grad-fh-beta-f}
    \end{subfigure}
    \hfill
    \begin{subfigure}[t]{0.3\textwidth}
        \includegraphics[width=\columnwidth]{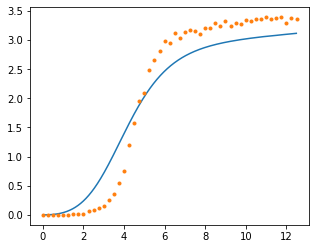}
        \caption{$D_{KL}(q_\beta\|p)$ ($y$ axis) vs $\beta$ ($x$ axis).}
        \label{fig:th-grad-fh-beta-kl}
    \end{subfigure}
    \hfill
    \begin{subfigure}[t]{0.3\textwidth}
        \includegraphics[width=\columnwidth]{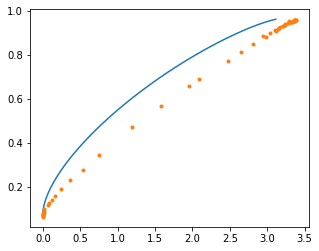}
        \caption{$\mathbb{E}_{q_\beta}\,f_h$ ($y$ axis) vs $D_{KL}(q_\beta\|p)$ ($x$ axis).}
        \label{fig:th-grad-fh-kl-f}
    \end{subfigure}
    \caption{Theoretical (plain line) and experimental (dashed line) estimates of $\mathbb{E}_{q_\beta}\,f_{h}$ and $D_{KL}(q_\beta\|p)$ {\em vs} $\beta$ for $f(x) = h(\ell=4|x)$.}
    \label{fig:th-grad-fh}
\end{figure}

\begin{figure}
    \centering
    \begin{subfigure}[t]{0.3\textwidth}
        \includegraphics[width=\columnwidth]{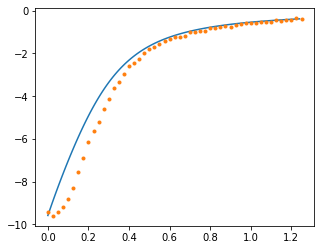}
        \caption{$\mathbb{E}_{q_\beta}\,f_{\log h}$ ($y$ axis) vs $\beta$ ($x$ axis).}
        \label{fig:th-grad-flogh-beta-f}
    \end{subfigure}
    \hfill
    \begin{subfigure}[t]{0.3\textwidth}
        \includegraphics[width=\columnwidth]{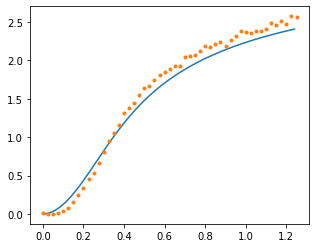}
        \caption{$D_{KL}(q_\beta\|p)$ ($y$ axis) vs $\beta$ ($x$ axis).}
        \label{fig:th-grad-flogh-beta-kl}
    \end{subfigure}
    \hfill
    \begin{subfigure}[t]{0.3\textwidth}
        \includegraphics[width=\columnwidth]{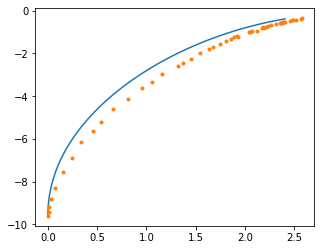}
        \caption{$\mathbb{E}_{q_\beta}\,f_{\log h}$ ($y$ axis) vs $D_{KL}(q_\beta\|p)$ ($x$ axis).}
        \label{fig:th-grad-flogh-kl-f}
    \end{subfigure}
    \caption{Theoretical (plain line) and experimental (dashed line) estimates of $\mathbb{E}_{q_\beta}\,f_{\log h}$ and $D_{KL}(q_\beta\|p)$ {\em vs} $\beta$ for $f(x) = \log h(\ell=4|x)$.}
    \label{fig:th-grad-flogh}
\end{figure}

For small values of $\beta$, with criterion $f_h$,  Fig. \ref{fig:th-grad-fh} shows that the empirical $\mathbb{E}_{q_\beta}\,f_h$ does not much increase, while $q_\beta$ remains close to $p$ ($D_{KL}(q_\beta\|p)$ stays close to 0). In other words, the bias seems ineffective. Quite the contrary, for large values of $\beta$, the empirical $D_{KL}(q_\beta\|p)$ increases significantly faster than the theoretical estimate;  \Method\ overshoots and focuses too much the support of distribution $q_\beta$. In comparison, a much smaller gap between the theoretical and empirical estimates is observed with criterion $f_{log.h}$ (Fig. \ref{fig:th-grad-flogh}).

These observations are in agreement with the analysis proposed in section 4.2: $f_h$ only provides useful gradients in the boundary of the targeted class. Accordingly, the process finds itself in one out of two stable states: doing nothing ($q_\beta = p$); or restricting the support of $q_\beta$ to that of the targeted class. \Method\ abruptly switches from the first to the second stable state (Fig. \ref{fig:th-grad-fh}.b), offering little control through $\beta$. 
When setting $f(x) = \log p(\ell|x))$ instead, $f(x)$ is less and less often saturated, enabling its gradient to provide smooth information. This information enables the user to finely control the bias through $\beta$, making the support of $q_\beta$ to gracefully tend toward the support of the targeted class.

\section{Generality of the approach: a proof of concept on MNIST}

\label{app:bold-mnist}

The claim is that \Method\ can be applied using any differentiable criterion (with exploitable gradient, see Appendix C. above). Three criteria are illustrated on Figs. \ref{fig:mnist-cls}, \ref{fig:mnist-pixels} and \ref{fig:mnist-neg-pixels}, respectively biasing the generative process toward a certain class, figures with more white pixels, or less white pixels.

The fine-grained control of the bias is illustrated on Fig. \ref{fig:mnist-cls} on MNIST, with target class $\ell = 4$, using a GAN model $p$. The Pareto front depicting the bi-criteria optimization trade-off ($\mathbb{E}_{q_\beta}f$ vs $D_{KL}(q_\beta \|\ p)$ for $\beta$ ranging from 0 to 2.5) is displayed on Fig. \ref{fig:mnist-cls-kl-f}, and the biased generated samples, where each row from top to bottom displays the samples generated with increasing values of $\beta$, are displayed on  Fig. \ref{fig:mnist-cls-images}.  
Indeed, class 4 is more prevalent as $\beta$ increases; class 9 is the last one to disappear, as being the most similar to the 4 one; for the highest values of $\beta$, only digits in class 4 are generated, yielding the same result as a conditional generative model, as expected.

A similar interpretation can be made for the two other examples on Figs. \ref{fig:mnist-pixels} and \ref{fig:mnist-neg-pixels}.

\begin{figure}[h]
    \centering
    \begin{subfigure}[t]{0.4\textwidth}
        \includegraphics[width=\columnwidth]{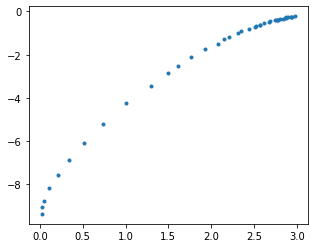}
        \caption{$\mathbb{E}_{q_\beta}\,f$ ($y$ axis) vs $D_{KL}(q_\beta\|p)$ ($x$ axis).}
        \label{fig:mnist-cls-kl-f}
    \end{subfigure}
    \hspace{1cm}
    \begin{subfigure}[t]{0.4\textwidth}
        \includegraphics[width=\columnwidth]{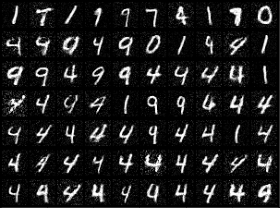}
        \caption{Generated samples, with the strength $\beta$ of the bias increasing from top to bottom rows.}
        \label{fig:mnist-cls-images}
    \end{subfigure}
    \caption{Using \Method\ to condition a generative model in the latent space, with $p$ a GAN trained on MNIST and $f = \log {\hat h(\text{class } 4\,|\,z)}$, and $h$ an independently trained classifier on the instance space. Left: Pareto front of both \objectives. Right: generated samples, with top to bottom rows respectively corresponding to $\beta$ in $\{0.0, 0.25, 0.5, 0.75, 1.0, 1.25\}$, and corresponding $D_{KL}$ values $0.0, 0.4, 1.3, 2.2, 2.5, 2.7$.}
    \label{fig:mnist-cls}
\end{figure}

\begin{figure}
    \centering
    \begin{subfigure}[t]{0.4\textwidth}
        \includegraphics[width=\columnwidth]{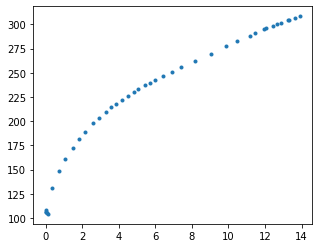}
        \caption{$\mathbb{E}_{q_\beta}\,f$ ($y$ axis) vs $D_{KL}(q_\beta\|p)$ ($x$ axis).}
        \label{fig:mnist-pixels-kl-f}
    \end{subfigure}
    \hspace{1cm}
    \begin{subfigure}[t]{0.4\textwidth}
        \includegraphics[width=\columnwidth]{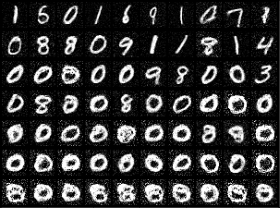}
        \caption{Generated samples, with the strength $\beta$ of the bias increasing from top to bottom rows.}
        \label{fig:mnist-pixels-images}
    \end{subfigure}
    \caption{Using \Method\ to condition a generative model in the latent space, with $p$ a GAN trained on MNIST and $f(x) = \sum_{i \in pixels} x_i$. Left: Pareto front of both \objectives. Right: generated samples, with top to bottom rows respectively corresponding to $\beta$ in $\{0.0, 0.025, 0.05, 0.075, 0.1, 0.125, 0.150\}$, and corresponding $D_{KL}$ values $0.0, 0.3, 2.1, 3.9, 5.4, 7.4, 8.7$.}
    \label{fig:mnist-pixels}
\end{figure}

\begin{figure}
    \centering
    \begin{subfigure}[t]{0.4\textwidth}
        \includegraphics[width=\columnwidth]{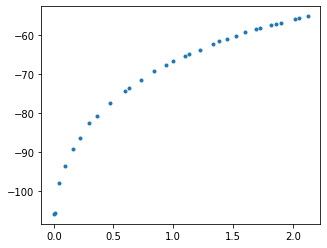}
        \caption{$\mathbb{E}_{q_\beta}\,f$ ($y$ axis) vs $D_{KL}(q_\beta\|p)$ ($x$ axis).}
        \label{fig:mnist-neg-pixels-kl-f}
    \end{subfigure}
    \hspace{1cm}
    \begin{subfigure}[t]{0.4\textwidth}
        \includegraphics[width=\columnwidth]{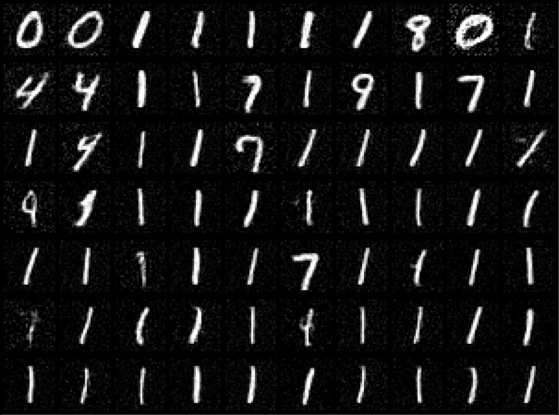}
        \caption{Generated samples, with the strength $\beta$ of the bias increasing from top to bottom rows.}
        \label{fig:mnist-neg-pixels-images}
    \end{subfigure}
    \caption{Using \Method\ to condition a generative model in the latent space, with $p$ a GAN trained on MNIST and $f(x) = \sum_{i \in pixels} x_i$. Left: Pareto front of both \objectives. Right: generated samples, with top to bottom rows respectively corresponding to $\beta$ in $\{0.0, 0.025, 0.05, 0.075, 0.1, 0.125, 0.150\}$, and corresponding $D_{KL}$ values $0.0, 0.2, 0.6, 1.1, 1.4, 1.8, 2.1$.}
    \label{fig:mnist-neg-pixels}
\end{figure}

As seen on Fig. \ref{fig:mnist-pixels}, biasing the generative model toward figures with more white pixels is achieved through controlling both the class of the generated figures (class $0$ and $8$) and the style of the generated numbers (with thick strokes). Quite the contrary, biasing the generative model toward figures with less white pixels results in generating very thin 1s.


\section{Rejection sampling with \Method\ in a real-world application}

\label{app:electrical}

As said, the application domain concerns smart grid management and dimensioning. The latter requires key indicators (consumption peak) to be estimated from consumption curves generated under diverse scenarii.
A versatile generative model is trained with a VAE, exploiting real weekly consumption curves aggregated over 10 households (thus with a higher variance compared to the curves aggregated over 100 households, considered in the main paper). 

The flexibility of the approach is demonstrated using several criteria. 

The first criterion aims to maximize the consumption over a particular day (here Wednesday, Fig. \ref{fig:elect-max-wed}). The goal is achieved by maximizing the weekly consumption, with the consumption on Wednesday being only slightly higher than the average one. 
The second criterion aims to
maximize the difference between the Wednesday consumption and the average weekly consumption (intuitively, this criterion corresponds to a worst-case analysis scenario). Using this criterion and allowing the $D_{KL}$ to take large values (corresponding to a rejection sampling with probability $10^{-4}$) yields the curves illustrated on  Fig. \ref{fig:elect-max-wed-avg}. Despite the strength of the bias, \Method\ still manages to generate diverse samples; furthermore, the sample variance is comparable to that of the original data.

\begin{figure}
    \centering
    \begin{subfigure}[t]{0.49\textwidth}
        \includegraphics[width=\columnwidth]{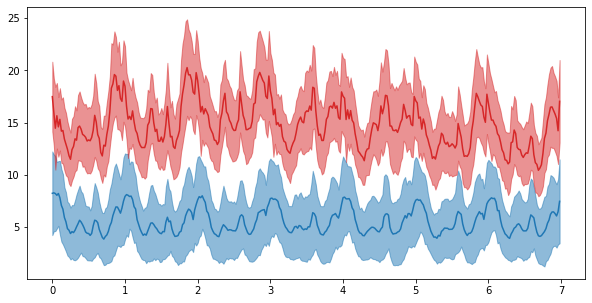}
        \caption{Maximizing the consumption of Wednesday.}
        \label{fig:elect-max-wed}
    \end{subfigure}
    \begin{subfigure}[t]{0.49\textwidth}
        \includegraphics[width=\columnwidth]{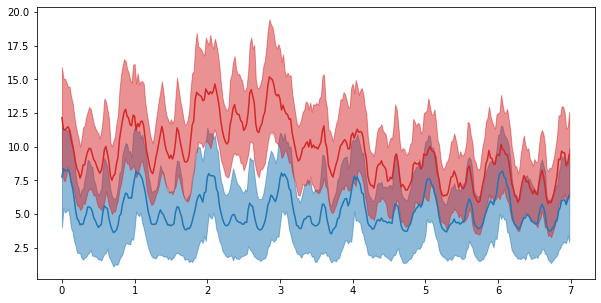}
        \caption{Maximizing the difference between the consumption of Wednesday and the average weekly consumption.}
        \label{fig:elect-max-wed-avg}
    \end{subfigure}
    \caption{Application of \Method\ in the context of energy management: generating consumption curves biased according to: Average Wednesday consumption (\ref{fig:elect-max-wed}); Average Wednesday consumption and difference between Wednesday consumption and average weekly consumption (\ref{fig:elect-max-wed-avg}). Blue curves represent the mean and standard deviation of samples from the original model, and red curves that of samples from the biased model (best seen in color).}
    \label{fig:elect-max-day}
\end{figure}

The third criterion is related to the variability of the demand, with a high impact on the required flexibility of electricity production. A relevant indicator, referred to as
MAE by abuse of the definition,
is the  amount of consumption that would need to be moved in order to make the consumption constant along time (with same overall consumption), i.e. the L1 distance between the actual consumption curve and the flat curve with same overall consumption. Fig. \ref{fig:elect-l1} displays average generated consumption curves when  applying \Method\ to maximize or minimize the MAE.

\begin{figure}
    \centering
    \begin{subfigure}[t]{0.49\textwidth}
        \includegraphics[width=\columnwidth]{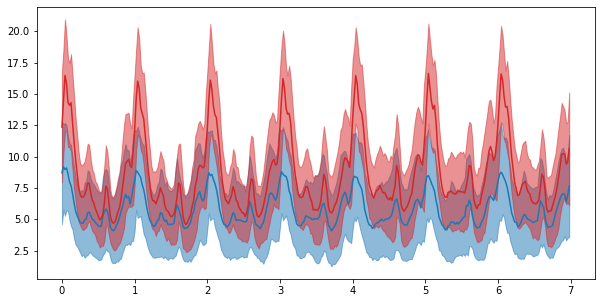}
        \caption{Maximizing the L1 distance to mean consumption.}
        \label{fig:elect-max-l1}
    \end{subfigure}
    \begin{subfigure}[t]{0.49\textwidth}
        \includegraphics[width=\columnwidth]{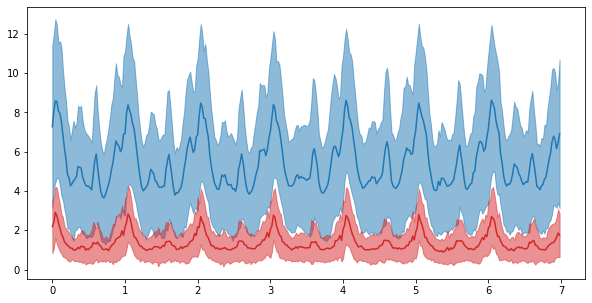}
        \caption{Minimizing the L1 distance to mean consumption.}
        \label{fig:elect-min-l1}
    \end{subfigure}
    \caption{Application of \Method\ in the context of energy management: generating consumption curves biased to maximize (\ref{fig:elect-max-l1}) and minimize (\ref{fig:elect-min-l1}) the L1 distance between the consumption and its average. Blue curves represent the mean and standard deviation of samples from the original model, and red curves that of samples from the biased model (best seen in color).}
    \label{fig:elect-l1}
\end{figure}

The curves obtained when minimizing the MAE (Fig. \ref{fig:elect-min-l1}) can be interpreted intuitively as: a good way to get a flat consumption curve is when the house is empty (e.g. during holidays), since inhabited houses typically present strong cyclical patterns across the day. 

When maximizing the MAE (Fig. \ref{fig:elect-max-l1}), the interpretation of the obtained curves is equally straightforward: \Method\ takes advantage of the natural variability of the data  to significantly increase the height of the consumption peaks, while only slightly increasing the average consumption, thereby yielding a high variance of the daily consumption.

\section{Adversarially refining a generative model using \Method: Discussion}

\label{app:adversarial}

A possible usage of \Method\ is to focus an overly general generative model (with support covering the data support) along an adversarial scheme, using a discriminator trained to distinguish between the actual and the generated samples as criterion $f$. 

Experiments are conducted to examine the feasibility of this 2-step generative modelling approach, with $p$ a VAE trained on MNIST and $g$ a classifier trained to discriminate the actual and the generated data (with accuracy $0.99$),
using its pre-activation output as $f$. \Method\ is applied on the VAE's latent space, and results are displayed on Fig. \ref{fig:mnist-disc}.

\begin{figure}
    \centering
    \begin{subfigure}[t]{0.4\textwidth}
        \includegraphics[width=\columnwidth]{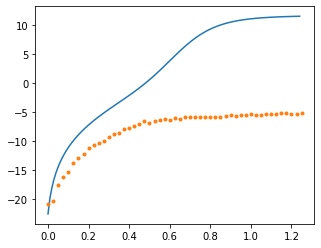}
        \caption{$\mathbb{E}_{q_\beta}\,f$ ($y$ axis) vs $\beta$ ($x$ axis).}
        \label{fig:mnist-disc-beta-f}
    \end{subfigure}
    \hspace{1cm}
    \begin{subfigure}[t]{0.4\textwidth}
        \includegraphics[width=\columnwidth]{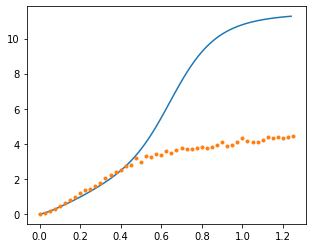}
        \caption{$D_{KL}(q_\beta\|p)$ ($y$ axis) vs $\beta$ ($x$ axis).}
        \label{fig:mnist-disc-beta-kl}
    \end{subfigure}

    \begin{subfigure}[t]{0.4\textwidth}
        \includegraphics[width=\columnwidth]{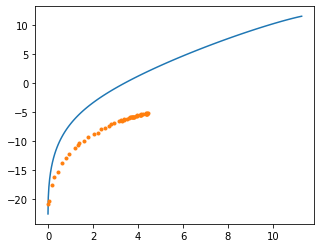}
        \caption{$\mathbb{E}_{q_\beta}\,f$ ($y$ axis) vs $D_{KL}(q_\beta\|p)$ ($x$ axis).}
        \label{fig:mnist-disc-kl-f}
    \end{subfigure}
    \hspace{1cm}
    \begin{subfigure}[t]{0.4\textwidth}
        \includegraphics[width=\columnwidth]{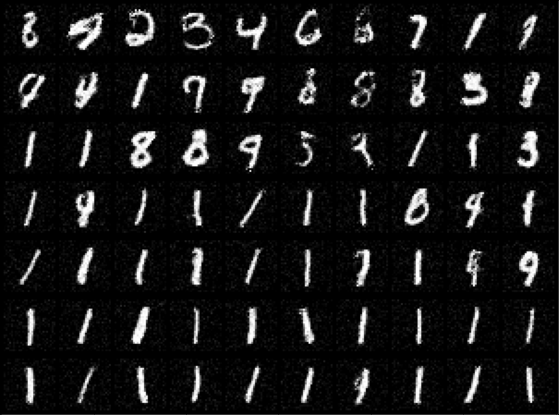}
        \caption{Generated samples, with the strength $\beta$ of the bias increasing from top to bottom rows.}
        \label{fig:mnist-disc-images}
    \end{subfigure}
    \caption{\Method: adversarial refinement of $p$ (VAE trained on MNIST) along criterion $f$, with $f$ a discriminator. As in Appendix \ref{app:grad},  Blue line are the theoretical curves, and orange dots are the empirical values. \ref{fig:mnist-disc-beta-f}, \ref{fig:mnist-disc-beta-kl}: Evolution of $\mathbb{E}_{q_\beta}\,f$ and $D_{KL}(q_\beta\|p)$ with $\beta$ . \ref{fig:mnist-disc-kl-f}: Pareto front of both \objectives. \ref{fig:mnist-disc-images}: generated samples with a clear mode dropping phenomenon, with top to bottom rows respectively corresponding to $\beta$ in $\{0.0, 0.125, 0.250, 0.375, 0.5, 0.625, 0.750\}$  and corresponding $D_{KL}$ values $0.0, 0.6, 1.4, 2.4, 2.9, 3.6, 3.7$. 
}
    \label{fig:mnist-disc}
\end{figure}

With same methodology as in Appendix \ref{app:grad}, the optimization process is assessed by comparing the theoretical and the empirical estimates of $\mathbb{E}_{q_\beta}\,f$ and $D_{KL}(q_\beta\|p)$.

The optimization fails: for $\beta \geq 0.5$, the $D_{KL}$ stagnates, that is, \Method\ cannot push $q_\beta$ farther away from $p$. For $\beta < 0.5$, \Method\ does not manage to increase $\mathbb{E}_{q_\beta}\,f$ as expected from the theoretical estimate. 

This change of behavior around $\beta = .5$ 
is analyzed in relation with the distribution of $f$ gradients wrt to $p$ (Fig.  \ref{fig:mnist-disc-grad}), involving most gradient norms in a reasonable range ($[0;10]$), while some gradients do explode with a norm as large as 230. This suggests that the optimization landscape includes large smooth regions with some very sharp regions (cliffs).

It is noted that at the change point ($\beta \approx .5$), $D_{KL} \approx 4$, that is, $q_\beta$ is focused on approximately 2\% of the support of $p$. Our interpretation is that, at this point the process meets the high gradient norm region and remains stuck. 

The fact that \Method\ cannot thus refine $p$ using the adversarial criterion is eventually blamed on two factors. Firstly, the discriminator seems sufficiently powerful to characterize the support of the true data as a set of isolated regions separated by high cliffs. Secondly, the generative model search space (based on Normalizing Flows; specifically, 6 Inverse AutoRegressive flows layers, each consisting of 4 fully-connected layers) seems not flexible enough to comply with the discriminator, and to approximate a mixture. Eventually,  \Method\ is unable to modify the structure of $p$ as desired in the small $\beta$ region (with $\mathbb{E}_{q_\beta}\,f$ about twice smaller than the theoretical estimate); and totally unable to modify it for $\beta > .5)$. 
How to remedy both limitations is left for future work.

\begin{figure}
    \centering
    \includegraphics[width=\columnwidth]{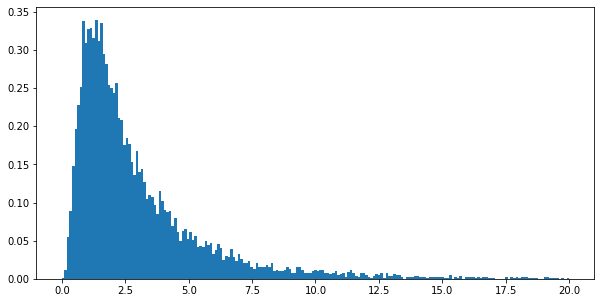}
    \caption{Distribution of the norm of the gradient of the objective $f$ (pre-activation output of the discriminator) wrt to the latent variable. The histogram is truncated at a norm of $20$ for legibility, but around $1\%$ of the gradients have a higher norm, going up to $230$. }
    \label{fig:mnist-disc-grad}
\end{figure}

\end{document}